\documentclass[runningheads]{llncs}
\usepackage{makeidx}
\usepackage{graphicx}
\usepackage{amsmath,amssymb} % define this before the line numbering.

\usepackage{color}
\usepackage{url}
\usepackage{amsfonts}    
\usepackage{nicefrac}       % compact symbols for 1/2, etc.
\usepackage{microtype}      % microtypography
\usepackage{xcolor}
\usepackage{grffile}
\usepackage{verbatim}
\usepackage{setspace}
\usepackage[hidelinks]{hyperref}
\usepackage[font=scriptsize,skip=2pt]{subcaption}
\usepackage{booktabs,tabularx}
\usepackage[inline]{enumitem}
\usepackage{mathtools}
\usepackage{wrapfig}
\usepackage[capitalize]{cleveref}
\usepackage{siunitx}
\newcommand{\myparagraph}[1]{\smallskip\noindent\textbf{#1}}
\newcommand{\eg}{\text{e}.\text{g}. }
\newcommand{\ie}{\text{i}.\text{e}. }
\newcommand{\owlog}{\text{w}.\text{l}.\text{o}.\text{g}. }

\begin{document}
	\pagestyle{headings}
	\mainmatter

	% Insert your submission number here
	\def\GCPR20SubNumber{67}

	% Replace with your title
	\title{Haar Wavelet based Block Autoregressive Flows for Trajectories}

	% DO NOT MODIFY these for the draft version that is used for the
	% review process.
	\titlerunning{Haar Wavelet based Block Autoregressive Flows for Trajectories}
	\authorrunning{A.\ Bhattacharyya et al.}
	%\author{Anonymous DAGM GCPR 2020 submission}
	%\institute{Paper ID \GCPR20SubNumber}
    
    \author{%
    Apratim Bhattacharyya \inst{1}  \and 
    Christoph-Nikolas Straehle \inst{2} \and 
    Mario Fritz \inst{3} \and \\  
    Bernt Schiele \inst{1}
    }%
    \institute{
    Max Planck Institute for Informatics, Saarbr\"{u}cken, Germany,\\
    \email{abhattac@mpi-inf.mpg.de}
    \and
    Bosch Center for Artificial Intelligence, Renningen, Germany
    \and
    CISPA Helmholtz Center for Information Security, Saarbr\"{u}cken, Germany}

	\maketitle

	\begin{abstract}
	    Prediction of trajectories such as that of pedestrians is crucial to the performance of autonomous agents. 
	    While previous works have leveraged conditional generative models like GANs and VAEs for learning the likely future trajectories, accurately modeling the dependency structure of these multimodal distributions, particularly over long time horizons remains challenging.
	    Normalizing flow based generative models can model complex distributions admitting exact inference. These include variants with split coupling invertible transformations that are easier to parallelize compared to their autoregressive counterparts.
	    To this end, we introduce a novel Haar wavelet based block autoregressive model leveraging split couplings, conditioned on coarse trajectories obtained from Haar wavelet based transformations at different levels of granularity. 
	    This yields an exact inference method that models trajectories at different spatio-temporal resolutions in a hierarchical manner. 
        We illustrate the advantages of our approach for generating diverse and accurate trajectories on two real-world datasets -- Stanford Drone and Intersection Drone.
        
		%Anticipation problems, such as prediction of pedestrian trajectories, can be cast in a conditional generative modeling framework. 
        %Despite the recent progress in deep conditional generative models, accurate predictive models for complex, real-world trajectory data has remained challenging.
        %However, accurately capturing long range dependencies as well as  modeling local and global structures remains challenging. 
	\end{abstract}

\section{Introduction}
\label{sec:introduction}
\begin{figure}[!b]
\centering
\includegraphics[height=2.6cm]{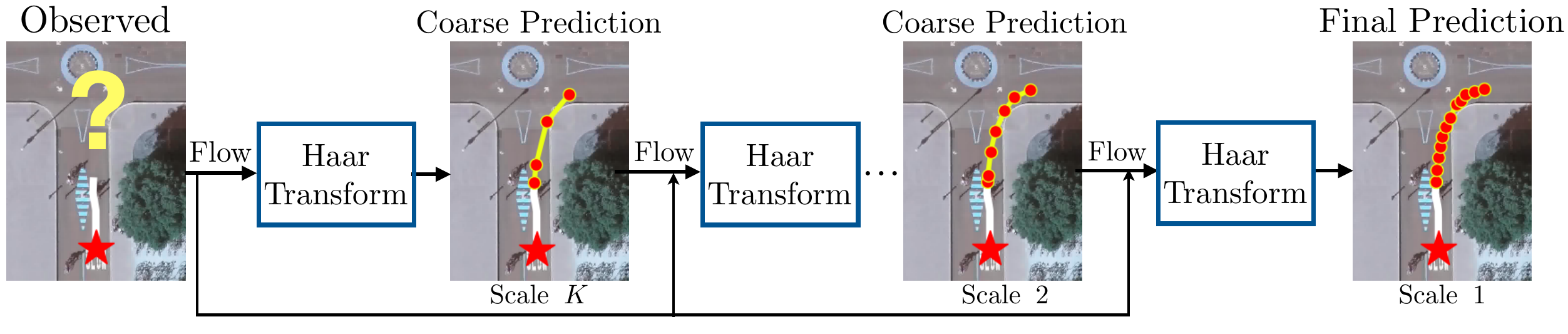}%
\caption{Our normalizing flow based model uses a Haar wavelet based decomposition to block autoregressively model trajectories at $K$ coarse-to-fine scales.}
\label{fig:teaser}
\end{figure}

Anticipation is a key competence for autonomous agents such as self-driving vehicles to operate in the real world. Many such tasks involving anticipation can be cast as trajectory prediction problems, \eg anticipation of pedestrian behaviour in urban driving scenarios.
To capture the uncertainty of the real world, it is crucial to model the distribution of likely future trajectories.
Therefore recent works \cite{bhattacharyya2019conditional,bhattacharyya2018accurate,lee2017desire,sadeghian2018sophie} have focused on modeling the distribution of likely future trajectories using either generative adversarial networks (GANs,  \cite{GoodfellowPMXWOCB14}) or variational autoencoders (VAEs, \cite{kingma2013auto}). 
However, GANs are prone to mode collapse and the performance of VAEs depends on the tightness of the variational lower bound on the data log-likelihood which is hard to control in practice \cite{cremer2018inference,huang2020augmented}.
This makes it difficult to accurately model the distribution of likely future trajectories. 

Normalizing flow based exact likelihood models \cite{dinh2014nice,dinh2016density,kingma2018glow} have  been considered to overcome these limitations of GANs and  VAEs in the context of image synthesis.
Building on the success of these methods, recent approaches have extended the flow models for density estimation of sequential data \eg video \cite{kumar2019videoflow} and audio \cite{kim2019flowavenet}. 
Yet, VideoFlow \cite{kumar2019videoflow} is autoregressive in the temporal dimension which results in the prediction errors accumulating over time \cite{lee2018stochastic} and  reduced efficiency in sampling. 
Furthermore, FloWaveNet \cite{kim2019flowavenet} extends flows to audio sequences with odd-even splits along the temporal dimension, encoding only \emph{local} dependencies \cite{marscf20cvpr,huang2020augmented,kirichenko2020normalizing}.
We address these challenges of flow based models for trajectory generation  and develop an exact inference framework to accurately model future trajectory sequences by harnessing long-term spatio temporal structure in the underlying trajectory distribution.

In this work, we propose \emph{HBA-Flow}, an exact inference model with coarse-to-fine block autoregressive structure to encode long term spatio-temporal correlations for multimodal trajectory prediction.   
The advantage of the proposed framework is that multimodality can be captured over long time horizons by sampling trajectories at coarse-to-fine spatial and temporal scales (\cref{fig:teaser}).
%This is reflected in the trajectories generated by the proposed framework (\cref{fig:teaser}), 
Our contributions are:   
\begin{enumerate*}
    \item we introduce a block autoregressive exact inference model using Haar wavelets where flows applied at a certain scale are conditioned on coarse trajectories from  previous scale. The trajectories at each level are obtained after the application of Haar wavelet based transformations, thereby modeling long term spatio-temporal correlations.
     \item Our HBA-Flow  model, by virtue of block autoregressive structure, integrates a multi-scale block autoregressive prior which further improves modeling flexibility by encoding dependencies in the latent space.
    \item Furthermore,  we show that compared to fully autoregressive approaches \cite{kumar2019videoflow}, our HBA-Flow model is computationally more efficient as the number of sampling steps grows logarithmically in trajectory length.
    %computationally efficient compared to fully autoregressive approaches , as the number of sampling steps grows logarithmically in trajectory length.
    \item We demonstrate the effectiveness of our approach for trajectory prediction on Stanford Drone and Intersection Drone, with improved accuracy over long time horizons.
\end{enumerate*}

\section{Related Work}
\myparagraph{Pedestrian Trajectory Prediction.} Work on traffic participant prediction dates back to the Social Forces model \cite{helbing1995social}. More recent works \cite{alahi2016social,helbing1995social,yamaguchi2011you,robicquet2016learning} consider the problem of traffic participant prediction in a social context, by taking into account interactions among traffic participants. Notably, Social LSTM \cite{alahi2016social}  introduces a social pooling layer to aggregate interaction information of nearby traffic participants. An efficient extension of the social pooling operation is developed in \cite{deo2018convolutional} and alternate instance and category layers to model interactions in \cite{ma2019trafficpredict}. Weighted interactions are proposed in \cite{chandra2019traphic}. In contrast,  a multi-agent tensor fusion scheme is proposed in \cite{zhao2019multi} to capture interactions. An attention based model to effectively integrate visual cues in path prediction tasks is proposed in \cite{sadeghian2018car}. However, these methods mostly assume a deterministic future and do not directly deal with the challenges of uncertainty and multimodality.

\myparagraph{Generative Modeling of Trajectories.} To deal with the challenges of uncertainty and multimodality in anticipating future trajectories, recent works employ either conditional VAEs or GANs to capture the distribution of future trajectories. This includes, a conditional VAE based model with a RNN based refinement module \cite{lee2017desire}, a VAE based model \cite{felsen2018will} that ``personalizes'' prediction to individual agent behavior, a diversity enhancing ``Best of Many'' loss \cite{bhattacharyya2018accurate} to better capture multimodality with VAEs, an expressive normalizing flow based prior for conditional VAEs \cite{bhattacharyya2019conditional} among others. However, VAE based models only maximize a lower bound on the data likelihood, limiting their ability to effectively model trajectory data.
Other works, use GANs \cite{gupta2018social,zhao2019multi,sadeghian2018sophie} to generate socially compliant trajectories. GANs lead to missed modes of the data distribution.
Additionally, \cite{rhinehart2018r2p2,deo2019scene} introduce push-forward policies and motion planning for generative modeling of trajectories. Determinantal point processes are used in \cite{yuan2019diverse} to better capture diversity of trajectory distributions. The work of \cite{mangalam2020not} shows that additionally modeling the distribution of trajectory end points can improve accuracy. However, it is unclear if the model of \cite{mangalam2020not} can be used for predictions across variable time horizons.
In contrast to these approaches, in this work we directly maximize the exact likelihood of the trajectories, thus better capturing the underlying true trajectory distribution.

\myparagraph{Autoregressive Models.} Autoregressive exact inference models like PixelCNN \cite{OordKEKVG16} have shown promise in generative modeling.
Autoregressive models for sequential data includes  a convolutional autoregressive model \cite{vanwavenet} for raw audio and an autoregressive method for video frame prediction \cite{kumar2019videoflow}.
In particular, for sequential data involving trajectories, recent works \cite{pajouheshgar2018back} propose an autoregressive method based on visual sources. 
The main limitation of autoregressive approaches is that the models are difficult to parallelize. Moreover, in case of sequential data, errors tend to accumulate over time \cite{lee2018stochastic}.

\myparagraph{Normalizing Flows.} 
%Normalizing flows are exact inference models that transform a complex distribution to a simpler distribution through a series of invertible transformations. 
%\cite{PapamakariosMP17} develop an autoregressive invertible transformation based model with masked decoders. While allowing exact inference, these models are difficult to parallelize.
Split coupling normalizing flow models with affine transformations \cite{dinh2014nice} offer computationally efficient tractable Jacobians. 
Recent methods \cite{dinh2016density,kingma2018glow} have therefore focused on split coupling flows which are easier to parallelize. 
Flow models are extended in \cite{dinh2016density} to multiscale architecture and the modeling capacity of flow models is further improved in \cite{kingma2018glow} by introducing $1 \times 1$ convolution.
Recently, flow models with more complex invertible components \cite{chen2019residual,pmlr-v97-ho19a} have been leveraged for generative modeling of images.
Recent works like FloWaveNet \cite{kim2019flowavenet} and VideoFlow  \cite{kim2019flowavenet} adapt the multi-scale architecture of Glow \cite{kingma2018glow} with sequential latent spaces to model sequential data, for raw audio and video frames respectively. However, these models still suffer from the limited modeling flexibility of the split coupling flows. 
The ``squeeze'' spatial pooling operation in \cite{kingma2018glow} is replaced with a Haar wavelet based downsampling scheme in \cite{ardizzone2019guided} along the spatial dimensions. Although this leads to improved results on image data, this operation is not particularly effective in case of sequential data as it does not influence temporal receptive fields for trajectories -- crucial for modeling long-term temporal dependencies. Therefore, Haar wavelet downsampling of \cite{ardizzone2019guided} does not lead to significant improvement in performance on sequential data (also observed empirically). In this work, instead of employing Haar wavelets as a downsampling operation for reducing spatial resolution \cite{ardizzone2019guided} in split coupling flows, we formulate a coarse-to-fine block autoregressive model where Haar wavelets produce trajectories at different spatio-temporal resolutions.

%combination with normal split coupling flows for image generation. However they do not give the jacobian determinant of the transformation and thus can not compute the exact log likelihood of the model. In comparison, our generalized Haar wavelet in combination with the proposed Haar block autoregressive architecture better captures low frequency correlations in the data (Figure~\ref{fig:syn}).

\section{Block Autoregressive Modeling of Trajectories}
In this work, we propose a coarse-to-fine block autoregressive exact inference model, \emph{HBA-Flow}, for trajectory sequences.
%our HBA-Flow can accurately capture complex multi-modal trajectorey distributions as long-range spatio-temporal correlations are easier to model at coarser scales and the coarse trajectories provide global context for the subsequent finer scales in the block autoregressive setup.
We first provide an overview of conditional normalizing flows which form the backbone of our HBA-Flow model.
To extend normalizing flows for trajectory prediction, we introduce an invertible transformation based on Haar wavelets which decomposes trajectories into $K$ coarse-to-fine scales (\cref{fig:teaser}).
This is beneficial for expressing long-range spatio-temporal correlations as coarse trajectories provide global context for the subsequent finer scales.
Our proposed HBA-Flow framework integrates the coarse-to-fine transformations with invertible split coupling flows where it block autoregressively models the transformed trajectories at $K$ scales. 
%The use of the Haar wavelet transformation helps our HBA-Flow can accurately capture complex multi-modal trajectorey distributions as long-range spatio-temporal correlations are easier to model at coarser scales and the coarse trajectories provide global context for the subsequent finer scales in the block autoregressive setup. 

\subsection{Conditional Normalizing Flows for Sequential Data}
We base our {HBA-Flow} model on normalizing flows \cite{dinh2014nice} which are a type of exact inference model. 
In particular, we consider the transformation of the conditional distribution $p(\mathbf{y}|\mathbf{x})$ of trajectories $\mathbf{y}$ to a distribution $p(\mathbf{z}|\mathbf{x})$ over $\mathbf{z}$ with conditional normalizing flows \cite{ardizzone2019guided,bhattacharyya2019conditional} using a sequence of $n$ transformations $g_{i}: \mathbf{h}_{i-1} \mapsto \mathbf{h}_{i}$, with $\mathbf{h}_0=\mathbf{y}$ and  parameters $\theta_i$,

\begin{equation}\label{eq:unflowtrans}
\begin{aligned}
\mathbf{y} \overset{g_{1}}{\longleftrightarrow} \mathbf{h}_{1} \overset{g_{2}}{\longleftrightarrow} \mathbf{h}_{2} \cdots \overset{g_{n}}{\longleftrightarrow} \mathbf{z}.
\end{aligned}
\end{equation}

Given the Jacobians $\mathbf{J}_{\theta_i} = \nicefrac{\partial \mathbf{h}_{i}}{\partial \mathbf{h}_{i-1}}$ of the transformations $g_{i}$, the exact likelihoods can be computed with the change of variables formula,

\begin{equation}\label{eq:flow}
\begin{aligned}
\log p_{\theta}(\mathbf{y} | \mathbf{x}) = \log p(\mathbf{z}|\mathbf{x}) + \sum\limits_{i=1}^{n} \log \;\lvert \det \mathbf{J}_{\theta_i} \lvert,
\end{aligned}
\end{equation}

Given that the density $p(\mathbf{z}|\mathbf{x})$ is known, the likelihood over $\mathbf{y}$ can be computed exactly.
Recent works \cite{dinh2014nice,dinh2016density,kingma2018glow} consider invertible split coupling transformations $g_{i}$ as they provide a good balance between efficiency and modeling flexibility. 
%In contrast, masked autoregressive transformations are difficult to parallelize \cite{PapamakariosMP17} and inverse autoregressive transformations do not perform well in practice \cite{ziegler2019latent}. 
In (conditional) split coupling transformations, the input $\mathbf{h}_{i}$ is split into two halves $ \mathbf{l}_{i}, \;\mathbf{r}_{i} $, and $g_{i}$ applies an invertible transformation only on $\mathbf{l}_{i}$ leaving $\mathbf{r}_{i}$ unchanged. 
The transformation parameters of $\mathbf{l}_{i}$ are dependent on $\mathbf{r}_{i}$ and $\mathbf{x}$, thus $\mathbf{h}_{i+1} = [ g_{i+1}(\mathbf{l}_{i} | \mathbf{r}_{i}, \mathbf{x}),\mathbf{r}_{i} ]$. The main advantage of (conditional) split coupling flows is that both inference and sampling are parallelizable when the transformations  $g_{i+1}$ have an efficient closed form expression of the inverse $g_{i+1}^{-1}$, \eg affine \cite{kingma2018glow} or non-linear squared \cite{ziegler2019latent} and unlike residual flows \cite{chen2019residual}. 
%Although, this limits the class of possible transformations that can be modeled with the split coupling flows, it allows for efficient inference and sampling. %In our \emph{HBA-Flow} model, we consider the non-linear squared split coupling layers of \cite{bhattacharyya2019conditional,ziegler2019latent}.
%These non-linear squared couplings provide added modeling flexibility over affine couplings used in \cite{dinh2016density,kim2019flowavenet,kingma2018glow,kumar2019videoflow}. Specifically, these non-linear squared couplings are better at capturing complex multi-modal distributions compared to affine couplings.
%However, split coupling flows have limited modeling flexibility \cite{PapamakariosMP17,ziegler2019latent} over autoregressive approaches. This is because only half the dimensions $\mathbf{l}_{i}$ are transformed at a time which makes it difficult to effectively capture correlations in the data. The transformation $f_{i+1}(\mathbf{l}_{i} | \mathbf{r}_{i})$ is further limited by the requirement of invertibility.
%As transformation $f_{i+1}$ must have a closed form expression of the inverse $f_{i+1}^{-1}$, this severely limiting the class of possible transformations that can be modeled with the split coupling flows.

As most of the prior work, \eg \cite{ardizzone2019guided,dinh2014nice,dinh2016density,kingma2018glow}, considers split coupling flows $g_{i}$ that are designed to deal with fixed length data, these models are not directly applicable to data of variable length such as trajectories. 
Moreover, recall that for variable length sequences, while VideoFlow \cite{kumar2019videoflow} utilizes split coupling based flows to model the distribution at each time-step, 
it is still fully autoregressive in the temporal dimension, thus offering limited computational efficiency. 
FloWaveNets \cite{kim2019flowavenet} split $\mathbf{l}_{i}$ and $\mathbf{r}_{i}$ along even-odd time-steps for audio synthesis. This even-odd formulation of the split operation along with the inductive bias \cite{kirichenko2020normalizing,huang2020augmented,marscf20cvpr} of split coupling based flow models is limited when expressing local and global dependencies which are crucial for capturing multimodality of the trajectories over long time horizons.
%This makes the model of \cite{kumar2019videoflow} difficult to parallelize unlike the split coupling flows of  \cite{ardizzone2019guided,dinh2014nice,dinh2016density,kingma2018glow}. 
%In contrast, FloWaveNets \cite{kim2019flowavenet} extend flows to the task of audio synthesis by introducing non-causal WaveNets in the split coupling transformations $g_{i+1}(\mathbf{l}_{i} | \mathbf{r}_{i}, \mathbf{x})$, allowing the input $\mathbf{l}_{i}, \;\mathbf{r}_{i}$ to be of arbitrary length. The split into $\mathbf{l}_{i}, \;\mathbf{r}_{i}$ is along even-odd time-steps. This even-odd formulation of the split operation along with the inductive bias \cite{kirichenko2020normalizing,huang2020augmented,marscf20cvpr} of split coupling based flow models encourages modeling of  local correlations only. However, long-term spatio-temporal correlations are crucial for accurate predictions over long time horizons. 
Next, we introduce our invertible transformation based on Haar wavelets to model trajectories at various coarse-to-fine levels to address the shortcomings of prior flow based methods \cite{kumar2019videoflow,kim2019flowavenet} for sequential data.
%\begin{wrapfigure}[15]{r}{4.5cm}
\begin{figure}[t!]
\begin{subfigure}{0.5\textwidth}
\centering
   \includegraphics[width=0.55\linewidth]{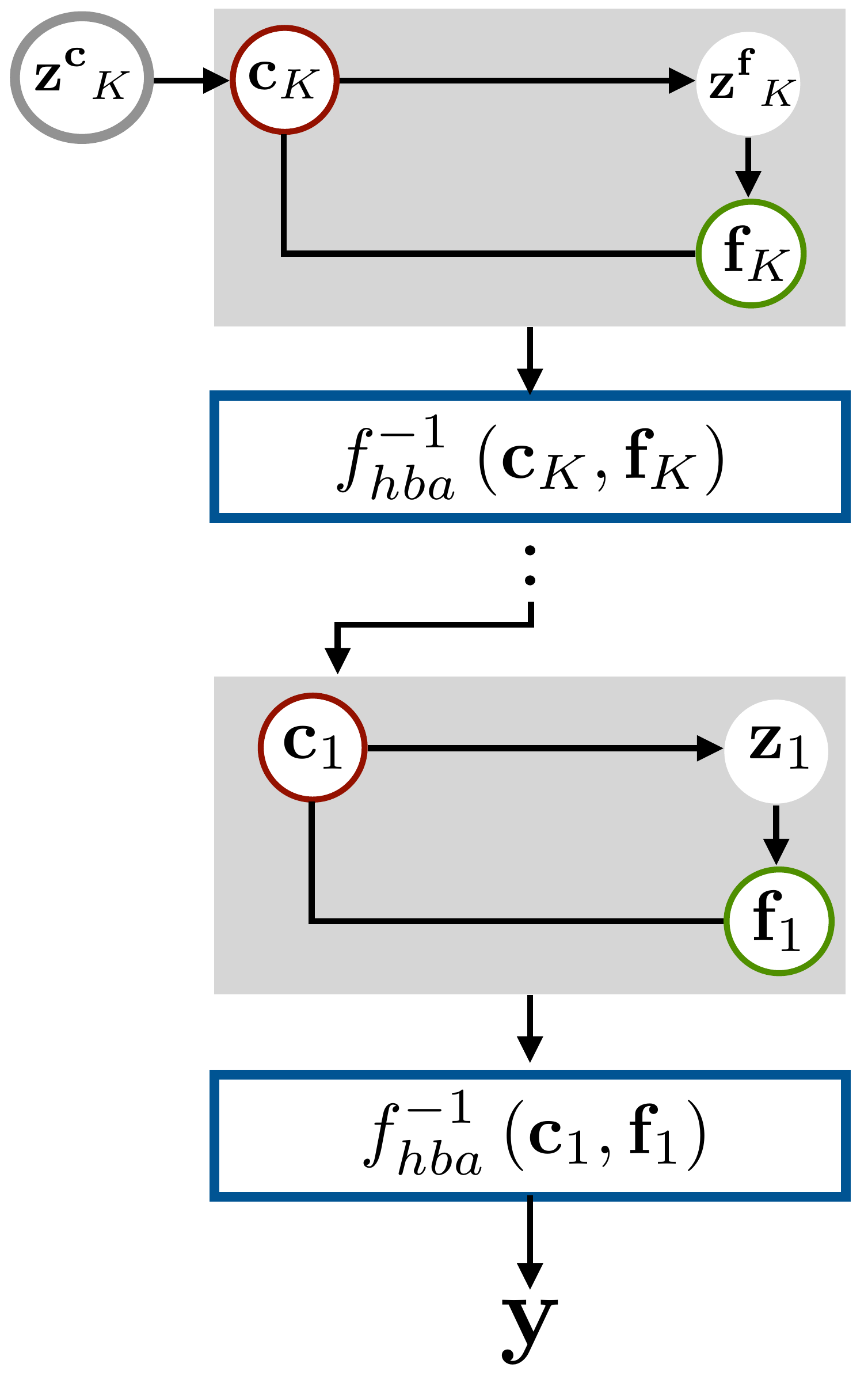}
   %\caption{Generative model for \emph{HBA-Flows}.}
   \label{fig:gen_model}
\end{subfigure}%
\hfill
\begin{subfigure}{0.5\textwidth}
\centering
  \includegraphics[width=0.75\linewidth]{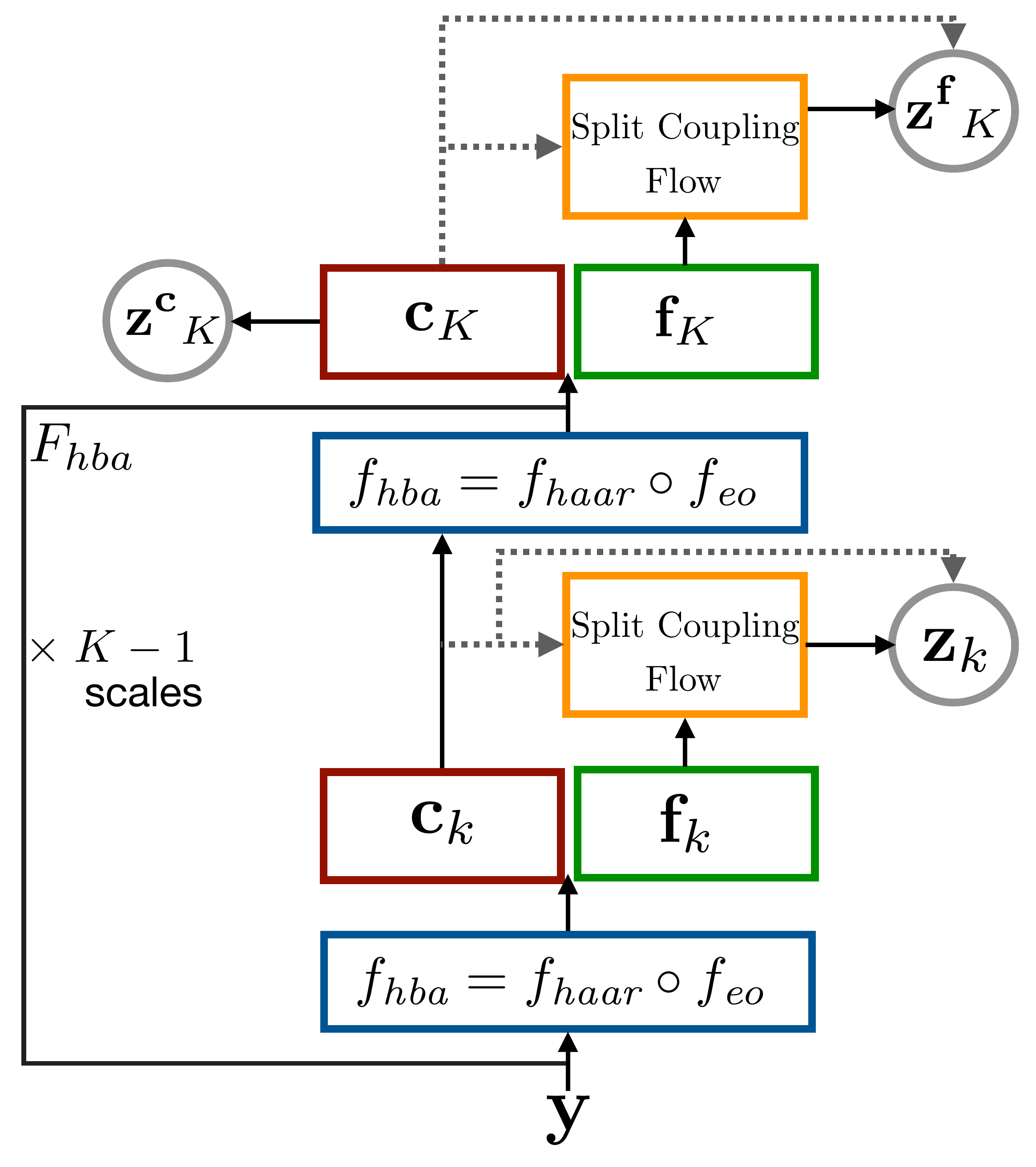}
\end{subfigure}
\caption{Left: \emph{HBA-Flow} generative model with the Haar wavelet \cite{haar1910haar} based representation $F_\textit{hba}$. Right: Our multi-scale \emph{HBA-Flow} model with $K$ scales of Haar based transformation.}\label{fig:archi}
\end{figure}

\subsection{Haar Wavelet based Invertible Transform}
Haar wavelet transform allows for a simple and easy to compute coarse-to-fine frequency decomposed representation with a finite number of components unlike alternatives \eg Fourier transformations \cite{porwik2004haar}. In our HBA-Flow  framework, we construct a transformation $F_\textit{hba}$ comprising of mappings $f_\textit{hba}$ recursively applied across $K$ scales. With this transformation, trajectories can be encoded at different levels of granularity along the temporal dimension.
We now formalize invertible function $f_\textit{hba}$ and its multi-scale Haar wavelet based  composition $F_\textit{hba}$.
%. The transformed trajectories $F_\textit{hba}(\mathbf{y})$ serve as input to our HBA-Flow model.

\myparagraph{Single Scale Invertible Transformation.}
Consider the trajectory at scale $k$ as $\mathbf{y}_k = [{\mathbf{y}}^1_k,\cdots,{\mathbf{y}}^{T_k}_k ]$, where $T_k$ is the number of timesteps of trajectory $\mathbf{y}_k$. Here, at scale $k=1$, $\mathbf{y_1}= \mathbf{y}$ is the input trajectory.
Each element of the trajectory is a vector, ${\mathbf{y}}^j_k \in \mathbb{R}^d$ encoding spatial information of the traffic participant. 
Our proposed invertible transformation $f_{hba}$ at any scale $k$ is a composition, $f_\textit{hba} = f_\textit{haar} \circ f_\textit{eo}$.
First,  $f_{eo}$ transforms the trajectory into even $(\mathbf{e}_k)$ and odd $(\mathbf{o}_k)$ downsampled trajectories,

\begin{equation}\label{eq:htr1}
\begin{aligned}
f_{eo}(\mathbf{y}_k) = \mathbf{e}_k, \mathbf{o}_k  \,\, \text{where}, \, \mathbf{e}_k = [{\mathbf{y}}^2_k, \cdots,{\mathbf{y}}^{T_k}_k ] \,\, \text{and} \,\,  \mathbf{o}_k = [{\mathbf{y}}^1_k, \cdots,{\mathbf{y}}^{T_k-1}_k ].
\end{aligned}
\end{equation}

Next, $f_{haar}$ takes as input the even $(\mathbf{e}_k)$ and odd $(\mathbf{o}_k)$ downsampled trajectories and transforms them into coarse ($\mathbf{c}_k$) and fine ($\mathbf{f}_k$) downsampled trajectories using a scalar ``mixing'' parameter $\alpha$. In detail,

\begin{equation}\label{eq:htr2}
\begin{aligned}
f_\textit{haar}( \mathbf{e}_k, \mathbf{o}_k ) = \mathbf{f}_k, \mathbf{c}_k  \,\, \text{where}, \, &\mathbf{c}_k = (1 - \alpha) \mathbf{e}_k + 
\alpha \mathbf{o}_k  \,\,\,\,\, \text{and} \\ 
& \mathbf{f}_k = \mathbf{o}_k - \mathbf{c}_k =  (1 - \alpha) \mathbf{o}_k +  (\alpha - 1) \mathbf{e}_k
\end{aligned}
\end{equation}

where, the coarse ($\mathbf{c}_k$) trajectory is the element-wise weighted average of the even $(\mathbf{e}_k)$ and odd $(\mathbf{o}_k)$ downsampled trajectories and the fine ($\mathbf{f}_k$) trajectory is the element-wise difference to the coarse downsampled trajectory. 
The coarse trajectories ($\mathbf{c}_k$) provide global context for finer scales in our block autoregressive approach, while the fine trajectories ($\mathbf{f}_k$) encode details at multiple scales.
We now discuss the invertibilty of this transformation $f_{hba}$ and compute the Jacobian.
\begin{lemma}\label{lemma1}
The generalized Haar transformation $f_\textit{hba} = f_\textit{haar} \circ f_\textit{eo}$ is invertible for $\alpha \in [0,1)$ and the determinant of the Jacobian of the transformation $f_\textit{hba} = f_\textit{haar} \circ f_{eo}$ for sequence of length $T_{k}$ with ${\mathbf{y}}^j_k \in \mathbb{R}^d$ is  $\det  \mathbf{J}_\textit{hba} = (1 - \alpha)^{\nicefrac{(d \cdot T_k)}{2}}$. 
\end{lemma}

We provide the proof in Appendix A.
\iffalse
First, note that $f_{haar}$ in (\ref{eq:htr2}) is a linear system. To compute the Jacobian of $f_{hba}$, note that each element of the output fine ($\mathbf{f}_k$) and coarse ($\mathbf{c}_k$) trajectories can be equivalently  written (using (\ref{eq:htr1}) and (\ref{eq:htr2})) in terms of the elements of the input trajectory $\mathbf{y}_k$. We can now rearrange the output by placing elements from  $\mathbf{f}_k$ and $\mathbf{c}_k$ in an alternating fashion.
%\begin{align*}
%f_{hba}(\mathbf{y}_k) = [ (1 - \alpha) {\mathbf{y}}^1_k + (\alpha - 1) {\mathbf{y}}^2_k, \,\, \alpha {\mathbf{y}}^1_k + (1 - \alpha) {\mathbf{y}}^2_k, \,\, \cdots \,\, ,\\ (1 - \alpha) {\mathbf{y}}^{T_k -1}_k + (\alpha - 1) {\mathbf{y}}^{T_k}_k, \,\, \alpha {\mathbf{y}}^{T_k -1}_k + (1 - \alpha) {\mathbf{y}}^{T_k}_k ]
%\end{align*}
This results in a Jacobian $J_\textit{hba} \in \mathbb{R}^{d \cdot {T_k} \times d \cdot {T_k} }$ which is block diagonal, with a repeating block $\bar{J}_\textit{hba}$ of the form (more details in the supplementary section),
\begin{align*}
\bar{J}_{hba} = \begin{pmatrix} 
(1 - \alpha) & (\alpha - 1) \\ 
\alpha & (1 - \alpha) 
\end{pmatrix}
\end{align*}

This block repeats $\nicefrac{(d \cdot T_k)}{2} $ times in $J_{H}$ as the trajectory is of length $T_k$ and each element of the trajectory has $d$ dimensions. Therefore, the determinant of the Jacobian $J_{hba}$ is $(1 - \alpha)^{\nicefrac{( d \cdot T_k )}{2}}$.

For $\alpha \in [0,1)$ we see that $\det {J}_{hba} > 0$. Thus, the linear system $f_{haar}$ in (\ref{eq:htr2}) is non-singular and invertible.
\fi
This property allows our {HBA-Flow} model to exploit $f_\textit{hba}$ for spatio-temporal decomposition of the trajectories $\mathbf{y}$ while remaining invertible with a tractable Jacobian for exact inference. Next, we use this transformation $f_\textit{hba}$ to build the coarse-to-fine multi-scale Haar wavelet based transformation $F_\textit{hba}$ and discuss its properties. 

\myparagraph{Multi-scale Haar Wavelet based Transformation.} 
To construct our generalized Haar wavelet based transformation $F_\textit{hba}$, the mapping $f_{hba}$ is applied recursively at $K$ scales (\cref{fig:archi}, left). 
The transformation $f_{hba}$ at a scale $k$ applies a low and a high pass filter pair on the input trajectory $\mathbf{y}_k$ resulting in the coarse trajectory $\mathbf{c}_k$ and the fine trajectory $\mathbf{f}_k$ with high frequency details.
The coarse (spatially and temporally sub-sampled) trajectory ($\mathbf{c}_k$) at scale $k$ is then further decomposed by using it as the input trajectory $\mathbf{y}_{k+1} = \mathbf{c}_k$ to $f_{hba}$ at scale $k+1$. This is repeated at $K$ scales, resulting in the complete Haar wavelet transformation $F_{hba}(\mathbf{y}) = [\mathbf{f}_1, \cdots, \mathbf{f}_K, \mathbf{c}_K ]$ which captures details at multiple ($K$) spatio-temporal scales. The finest scale $\mathbf{f}_1$ models high-frequency spatio-temporal information of the trajectory $\mathbf{y}$. 
The subsequent scales $\mathbf{f}_k$ represent details at coarser levels, with $\mathbf{c}_K$ being the coarsest transformation which expresses the ``high-level'' spatio-temporal structure of the trajectory (\cref{fig:teaser}).

Next, we show that the number of scales $K$ in $F_{hba}$ is upper bounded by the logarithm of the length of the sequence.
This implies that $F_{hba}$, when integrated in the multi-scale block auto-regressive model provides a computationally efficient setup for generating trajectories. 
\begin{lemma}\label{lemma3}
The number of scales $K$ of the Haar wavelet based representation $F_{hba}$ is $K \leq \log(T_1)$, for an initial input sequence $\mathbf{y}_1$ of length $T_1$.
\end{lemma}
\begin{proof}
The Haar wavelet based transformation $f_{hba}$ halves the length of trajectory $\mathbf{y}_k$ at each level $k$. Thus, for an initial input sequence $\mathbf{y}_1$ of length $T_1$, the length of the coarsest level $K$ in $F_{hba}(\mathbf{y})$ is $\lvert \mathbf{c}_K \lvert = \nicefrac{T_1}{2^K} \geq 1$. Thus, $K \leq \log(T_1)$.
\end{proof}

\subsection{Haar Block Autoregressive Framework}
\label{section:HBAmodel}
\myparagraph{{HBA-Flow} model.} We illustrate our {HBA-Flow} model in \cref{fig:archi}. 
Our HBA-Flow model first transforms the trajectories $\mathbf{y}$ using $F_\textit{hba}$, where the invertible transform $f_{hba}$ is recursively applied on the input trajectory $\mathbf{y}$ to obtain $\mathbf{f}_k$ and $\mathbf{c}_k$ at scales $k \in \{1,\cdots,K\}$. Therefore, the log-likelihood of a trajectory $\mathbf{y}$ under our {HBA-Flow} model can be expressed using the change of variables formula as,

\begin{equation}
\begin{aligned}
\log(p_{\theta}(\mathbf{y} | \mathbf{x})) &= \log(p_{\theta}(\mathbf{f}_1,\mathbf{c}_1 | \mathbf{x})) + \log \;\lvert \det \left(\mathbf{J}_\textit{hba}\right)_1 \lvert\\
&= \log(p_{\theta}(\mathbf{f}_1,\cdots,\mathbf{f}_K,\mathbf{c}_K | \mathbf{x})) + \sum\limits_{i=1}^{K} \log \;\lvert\det \left(\mathbf{J}_\textit{hba}\right)_i \lvert.
\end{aligned}
\end{equation}

Next, our HBA-Flow model factorizes the distribution of fine trajectories \owlog such that $\mathbf{f}_k$ at level $k$ is conditionally dependent on the representations at scales $k+1$ to $K$,

\begin{equation}\label{eq:bafac}
\begin{aligned}
\log(p_{\theta}(\mathbf{f}_1,\cdots,\mathbf{f}_K,\mathbf{c}_K | \mathbf{x}))
&= \log(p_{\theta}( \mathbf{f}_1 | \mathbf{f}_2,\cdots,\mathbf{f}_K,\mathbf{c}_K, \mathbf{x})) + \cdots \\ &+  \log(p_{\theta}( \mathbf{f}_K | \mathbf{c}_K, \mathbf{x})) + \log( p_{\theta}(\mathbf{c}_K | \mathbf{x})).
\end{aligned}
\end{equation}

Finally, note that $[\mathbf{f}_{k+1},\cdots,\mathbf{f}_{K},\mathbf{c}_K]$ is the output of the (bijective) transformation $F_{\textit{hba}}(\mathbf{c}_{k})$ where $f_\textit{hba}$ is recursively applied to  $\mathbf{c}_{k}=\mathbf{y}_{k+1}$ at scales $\{k+1, \cdots ,K\}$. Thus HBA-Flow equivalently models $p_{\theta}( \mathbf{f}_k | \mathbf{f}_{k+1},\cdots,\mathbf{c}_K, \mathbf{x})$ as $p_{\theta}( \mathbf{f}_k | \mathbf{c}_{k}, \mathbf{x})$,

\begin{equation}\label{eq:cafac}
\begin{aligned}
\log(p_{\theta}(\mathbf{y} | \mathbf{x})) =& \log(p_{\theta}( \mathbf{f}_1 | \mathbf{c}_1, \mathbf{x})) + \cdots +  \log(p_{\theta}( \mathbf{f}_K | \mathbf{c}_K, \mathbf{x})) \\
&+ \log( p_{\theta}(\mathbf{c}_K | \mathbf{x})) + \sum\limits_{i=1}^{K} \log \; \lvert \det \left(\mathbf{J}_\textit{hba}\right)_i \lvert.
\end{aligned}
\end{equation}

Therefore, as illustrated in \cref{fig:archi} (right), our {HBA-Flow} models the distribution of each of the fine components $\mathbf{f}_k$ block autoregressively conditioned on the coarse representation $\mathbf{c}_k$ at that level. 
The distribution $p_{\theta}( \mathbf{f}_k | \mathbf{c}_k, \mathbf{x} )$ at each scale $k$ is modeled using invertible conditional split coupling flows (\cref{fig:archi}, right) \cite{kim2019flowavenet}, which transform the input distribution to the distribution over latent ``priors'' $\mathbf{z}_k$.
This enables our framework to model variable length trajectories.
The log-likelihood with our HBA-Flow approach can be expressed using the change of variables formula as,

\begin{equation}\label{eq:scfflow}
\begin{aligned}
    \log(p_{\theta}( \mathbf{f}_k | \mathbf{c}_k, \mathbf{x})) = \log(p_{\phi}( \mathbf{z}_k | \mathbf{c}_k, \mathbf{x})) + \log \; \lvert\det (\mathbf{J}_\textit{sc})_k \lvert
\end{aligned}
\end{equation}

where, $\log \; \lvert\det (\mathbf{J}_\textit{sc})_k \lvert$ is the log determinant of Jacobian $(\mathbf{J}_\textit{sc})_k$ of the split coupling flow at level $k$. Thus, the likelihood of a trajectory $\mathbf{y}$ under our HBA-Flow model can be expressed exactly using \cref{eq:cafac,eq:scfflow}. 

The key advantage of our approach is that after spatial and temporal downsampling of coarse scales, it is easier to model long-term spatio-temporal dependencies. Moreover, conditioning the flows at each scale on the coarse trajectory provides global context as the downsampled coarse trajectory effectively increases the spatio-temporal receptive field. This enables our HBA-Flows better capture multimodality in the distribution of likely future trajectories.

\myparagraph{HBA-Prior.}
%Prior work on exact inference models for sequences consider \eg{} inefficient fully autoregressive priors \cite{kumar2019videoflow} or simple conditional Gaussian priors \cite{kim2019flowavenet}. However, 
Complex multimodel priors can considerably increase the modeling flexibility of generative models \cite{bhattacharyya2019conditional,kim2019flowavenet,kumar2019videoflow}. The block autoregressive structure of our HBA-Flow model allows us introduce a Haar block autoregressive prior (HBA-Prior) over $\mathbf{z} = [ \mathbf{z}_1, \cdots, \mathbf{z^f}_K, \mathbf{z^c}_K ]$ in \cref{eq:scfflow}, where $\mathbf{z}_k$ is the latent representation for scales $k \in \{ 1, \cdots, K-1 \}$ and  $\mathbf{z^f}_K, \mathbf{z^c}_K$ are the latents for the coarse and fine representations scales $K$.
The log-likelihood of the prior factorizes as,

\begin{equation}\label{eq:prior1}
\begin{aligned}
\log(p_{\phi}(\mathbf{z}| \mathbf{x})) = \log(p_{\phi}( \mathbf{z}_1 | \mathbf{z}_2, \cdots, \mathbf{z^f}_K,\mathbf{z^c}_K, \mathbf{x})) + \cdots \\ +  \log(p_{\phi}( \mathbf{z^f}_{K} | \mathbf{z^c}_K, \mathbf{x})) + \log( p_{\phi}(\mathbf{z^c}_K | \mathbf{x})).
\end{aligned}
\end{equation}

Each coarse level representation $\mathbf{c}_k$ is the output of a bijective  transformation of the latent variables $[ \mathbf{z}_{k+1},\cdots,\mathbf{z^f}_{K}\,\mathbf{z^c}_{K}]$ through the invertible split coupling flows and the transformations $f_{hba}$ at scales $\{ {k+1},\cdots,{K}\}$. 
Thus, HBA-Prior models $p_{\phi}( \mathbf{z}_k | \mathbf{z}_{k+1}, \cdots, \mathbf{z^f}_K,\mathbf{z^c}_K, \mathbf{x})$ as $p_{\phi}( \mathbf{z}_k | \mathbf{c}_k,\mathbf{x})$ at every scale (\cref{fig:archi}, left). 
The log-likelihood of the prior can also be expressed as,

\begin{equation}\label{eq:prior2}
\begin{aligned}
\log(p_{\phi}(\mathbf{z}| \mathbf{x})) = \log(p_{\phi}( \mathbf{z}_1 | \mathbf{c}_1, \mathbf{x})) + \cdots +  \log(p_{\phi}( \mathbf{z}_{K-1} | \mathbf{c}_{K-1}, \mathbf{x})) \\+ \log(p_{\phi}( \mathbf{z^f}_{K} | \mathbf{c}_K, \mathbf{x})) + \log( p_{\phi}(\mathbf{z^c}_K | \mathbf{x})). 
\end{aligned}
\end{equation}

We model $p_{\phi}( \mathbf{z}_k | \mathbf{c}_{k}, \mathbf{x})$ as conditional normal distributions which are multimodal as a result of the block autoregressive structure. 
In comparison to the fully autoregressive prior in \cite{kumar2019videoflow}, our {HBA-Prior} is efficient as it requires only $\mathcal{O}(\log(T_1))$ sampling steps. 
%During sampling, we first sample $\mathbf{z^c}_K$ from a normal distribution, following which we sample $\mathbf{z}_k$ for all levels and then using invertible Haar transformations $f_{hba}^{-1}$ recursively generate coarse representations at higher levels to get the final trajectories (\cref{fig:archi}, left). 

\myparagraph{Analysis of Sampling Time.} From \cref{eq:bafac} and \cref{fig:archi} (left), our HBA-Flow model autoregressively factorizes across the fine components $\mathbf{f}_k$ at $K$ scales. From \cref{lemma3}, $K \leq \log(T_1)$. At each scale our HBA-Flow samples the fine components $\mathbf{f}_k$ using split coupling flows, which are easy to parallelize. Thus, given enough parallel resources, our {HBA-Flow} model requires maximum $K \leq \log(T_1)$ \ie{} $\mathcal{O}(\log(T_1))$ sampling steps and is significantly more efficient compared to fully autoregressive approaches \eg{}VideoFlow \cite{kumar2019videoflow}, which require $\mathcal{O}(T_1)$ steps.

\section{Experiments}
We evaluate our approach for trajectory prediction on two challenging real world datasets -- Stanford Drone  \cite{robicquet2016learning} and Intersection Drone \cite{inDdataset}. 
These datasets contain trajectories of traffic participants including pedestrians, bicycles, cars recorded from an aerial platform. 
The distribution of likely future trajectories is highly multimodal due to the complexity of the traffic scenarios \eg at intersections.

\myparagraph{Evaluation Metrics.} We are primarily interested in measuring the match of the learned distribution to the true distribution. Therefore, we follow \cite{bhattacharyya2019conditional,bhattacharyya2018accurate,lee2017desire,pajouheshgar2018back} and use Euclidean error of the top 10\% of samples (predictions) and the (negative) conditional log-likelihood (-CLL) metrics.
The Euclidean error of the top 10\% of samples measures the coverage of all modes of the target distribution and is relatively robust to random guessing as shown in \cite{bhattacharyya2019conditional}.  

\myparagraph{Architecture Details.} We provide additional architecture details in Appendix B.

\begin{table*}[!t]
\centering
\resizebox{\textwidth}{!}{
\begin{tabular}{lc@{\hskip 0.2cm}c@{\hskip 0.2cm}c@{\hskip 0.2cm}c@{\hskip 0.2cm}c@{\hskip 0.2cm}c@{\hskip 0.2cm}c}
\toprule
Method & Visual & Er $@$ 1sec & Er $@$ 2sec & Er $@$ 3sec &  Er $@$ 4sec & -CLL & Speed \\
\midrule
``Shotgun'' \cite{pajouheshgar2018back} &  -- & 0.7 & 1.7 &  3.0 &  4.5 & 91.6 & -- \\%\citep{pajouheshgar2018back}
DESIRE-SI-IT4 \cite{lee2017desire}  & \checkmark & 1.2 & 2.3 & 3.4 & 5.3 & -- & -- \\ %\citep{lee2017desire} 
STCNN \cite{pajouheshgar2018back} & \checkmark & 1.2   & 2.1  & 3.3 &  4.6 & -- &  -- \\ %\citep{pajouheshgar2018back}
BMS-CVAE \cite{bhattacharyya2018accurate} & \checkmark & 0.8  & 1.7 & 3.1 & 4.6 & 126.6 & 58  \\ %\citep{bhattacharyya2018accurate}
CF-VAE \cite{bhattacharyya2019conditional} & --  & \textbf{0.7} & 1.5 & 2.5 & 3.6 & 84.6 & 47 \\
CF-VAE \cite{bhattacharyya2019conditional} & \checkmark  & \textbf{0.7} & 1.5 & 2.4 & 3.5 & 84.1 & 88 \\
\midrule
Auto-regressive \cite{kumar2019videoflow} & -- & \textbf{0.7}  & 1.5 & 2.6 & 3.7 & 86.8 & 134 \\
FloWaveNet \cite{kim2019flowavenet} & -- & \textbf{0.7}  & 1.5 & 2.5 & 3.6 & 84.5 & \textbf{38} \\
FloWaveNet \cite{kim2019flowavenet} + HWD \cite{ardizzone2019guided}  & -- & \textbf{0.7}  & 1.5 & 2.5 & 3.6 & 84.4 & \textbf{38} \\
FloWaveNet \cite{kim2019flowavenet} & \checkmark & \textbf{0.7}  & 1.5 & 2.4 & 3.5 & 84.1 & 77 \\
\midrule
{HBA-Flow} (Ours) & -- & \textbf{0.7}  & 1.5 & 2.4 & 3.4 & 84.1 & 41  \\
{HBA-Flow} + Prior (Ours) & -- & \textbf{0.7}  & \textbf{1.4} & \textbf{2.3} & 3.3 & 83.4 & 43 \\
{HBA-Flow} + Prior (Ours) & \checkmark & \textbf{0.7}  & \textbf{1.4} & \textbf{2.3} & \textbf{3.2} & \textbf{83.1} & 81  \\

\bottomrule
\end{tabular}}
\caption{Five fold cross validation on the Stanford Drone dataset. Lower is better for all metrics. Visual refers to additional conditioning on the last observed frame. Top: state of the art, Middle: Baselines and ablations, Bottom: Our HBA-Flow.}
\label{tab:stanford_drone_cross}
\end{table*}
%Euclidean error at ($\nicefrac{1}{5}$) resolution.

\begin{figure}
\centering
\scriptsize
\begin{tabular}{cccc} 

\toprule
Observed & \shortstack{Mean Top 10\%\\ \textcolor[HTML]{4363D8}{B} - GT, \textcolor[HTML]{FFE119}{Y} -\cite{kim2019flowavenet}, \textcolor[HTML]{E6194B}{R} - Ours}   &  \shortstack{FloWaveNet
\cite{kim2019flowavenet}\\ Predictions} & \shortstack{\emph{HBA-Flows} (Ours)\\Predictions} \\
\midrule
\includegraphics[width=0.2\linewidth]{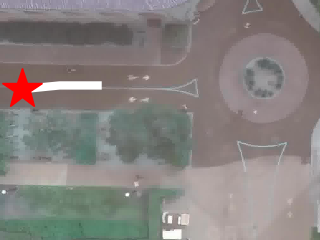} &
\includegraphics[width=0.2\linewidth]{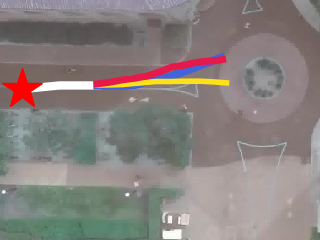} &
\includegraphics[width=0.2\linewidth]{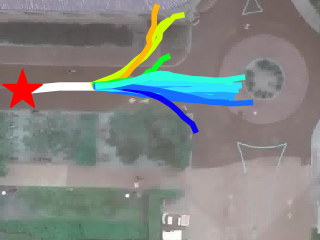} &
\includegraphics[width=0.2\linewidth]{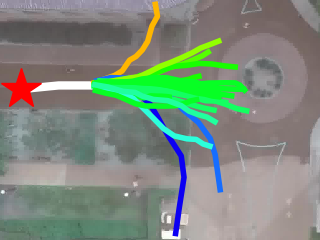}\\

\includegraphics[width=0.2\linewidth]{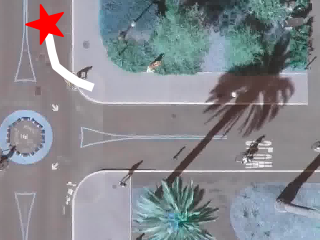} &
\includegraphics[width=0.2\linewidth]{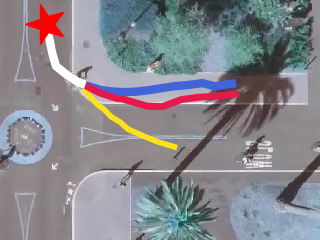} &
\includegraphics[width=0.2\linewidth]{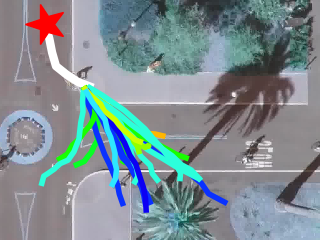} &
\includegraphics[width=0.2\linewidth]{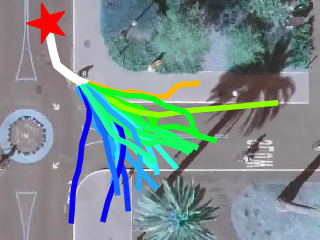}\\

\includegraphics[width=0.2\linewidth]{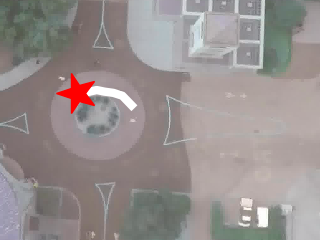} &
\includegraphics[width=0.2\linewidth]{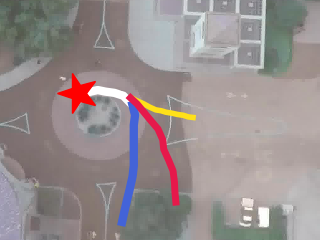} &
\includegraphics[width=0.2\linewidth]{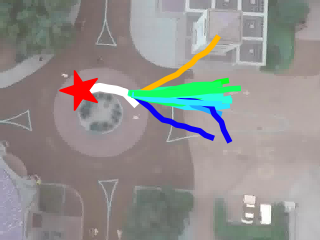} &
\includegraphics[width=0.2\linewidth]{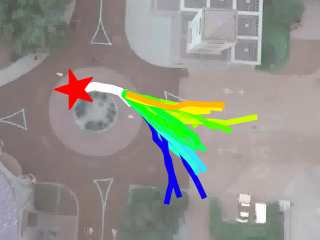}\\

\bottomrule

\end{tabular}
\caption{Mean top 10\% predictions (\textcolor[HTML]{4363D8}{Blue} - Groudtruth, \textcolor[HTML]{FFE119}{Yellow} - FloWaveNet \cite{kim2019flowavenet}, \textcolor[HTML]{E6194B}{Red} - Our \emph{HBA-Flow} model) and predictive distributions on Intersection Drone dataset.
The predictions of our {HBA-Flow} model are more diverse and better capture the multimodality the future trajectory distribution.}
\label{fig:stanford_drone_dist}
\end{figure}

\subsection{Stanford Drone}
%Here we evaluate on the real-world Stanford Drone dataset \cite{robicquet2016learning}, which consists of trajectories of traffic participants including pedestrians, bicycles, cars etc.\ recorded from an aerial platform in diverse scenarios. 
We use the standard five-fold cross validation evaluation protocol  \cite{bhattacharyya2019conditional,bhattacharyya2018accurate,lee2017desire,pajouheshgar2018back} and predict the trajectory up to 4 seconds into the future.
We use the Euclidean error of the top 10\% of predicted trajectories at the standard ($\nicefrac{1}{5}$) resolution using 50 samples and the CLL metric in \cref{tab:stanford_drone_cross}. We additionally report sampling time for a batch of 128 samples in milliseconds.

We compare our {HBA-Flow} model to the following state-of-the-art models:
The handcrafted ``Shotgun'' model \cite{pajouheshgar2018back}, the conditional VAE based models of \cite{bhattacharyya2018accurate,bhattacharyya2019conditional,lee2017desire} and the autoregressive STCNN model \cite{pajouheshgar2018back}.
We additionally include the various exact inference baselines for modeling trajectory sequences: the autoregressive flow model of VideoFlow \cite{kumar2019videoflow}, FloWaveNet \cite{kim2019flowavenet} (without our Haar wavelet based block autoregressive structure), FloWaveNet \cite{kim2019flowavenet} with the Haar wavelet downsampling of \cite{ardizzone2019guided} (FloWaveNet + HWD), our {HBA-Flow} model with a Gaussian prior (without our {HBA-Prior}).
The FloWaveNet \cite{kim2019flowavenet} baselines serves as ideal ablations to measure the effectiveness of our block autoregressive HBA-Flow model. For fair comparison, we use two scales (levels) $K=2$ with eight non-linear squared split coupling flows \cite{ziegler2019latent} each, for both our {HBA-Flow} and FloWaveNet \cite{kim2019flowavenet} models. 
%Similarly, we employ eight levels in case of the autoregressive flow VideoFlow baseline \cite{kumar2019videoflow}. 
Following \cite{bhattacharyya2019conditional,pajouheshgar2018back} we additionally experiment with conditioning on the last observed frame using a attention based CNN (indicated by ``Visual'' in \cref{tab:stanford_drone_cross}).

\begin{wraptable}[10]{r}{0.48\textwidth}
\centering
\scriptsize
\vspace{-2.0\baselineskip}
    \begin{tabular}{lcc}
    \toprule
    Method & mADE $\downarrow$ & mFDE $\downarrow$ \\
    \midrule
    SocialGAN \cite{gupta2018social} & 27.2 & 41.4\\
    MATF GAN \cite{zhao2019multi} & 22.5 & 33.5 \\
    SoPhie \cite{sadeghian2018sophie} & 16.2 &  29.3\\
    Goal Prediction \cite{deo2019scene} &  15.7 & 28.1 \\
    CF-VAE \cite{bhattacharyya2019conditional} & {12.6} & {22.3} \\
    \midrule
    HBA-Flow + Prior (Ours) & \textbf{10.8} & \textbf{19.8} \\
    \bottomrule
    \end{tabular}
    %\vspace{-0.3cm}
\caption{Evaluation on the Stanford Drone using the split of \cite{deo2019scene,sadeghian2018sophie,zhao2019multi}.}
\label{tab:stanford_drone_stand}
\end{wraptable}
We observe from \cref{tab:stanford_drone_cross} that our {HBA-Flow} model outperforms both state-of-the-art models and baselines. In particular, our {HBA-Flow} model outperforms the conditional VAE based models of \cite{bhattacharyya2019conditional,bhattacharyya2018accurate,lee2017desire} in terms of Euclidean distance and -CLL. Further, our HBA-Flow exhibits competitive sampling speeds. 
This shows the advantage of exact inference in the context of generative modeling of trajectories -- leading to better match to the groundtruth distribution.
Our {HBA-Flow} model generates accurate trajectories compared to the VideoFlow \cite{kumar2019videoflow} baseline. This is because unlike VideoFlow, errors do not accumulate in the temporal dimension of HBA-Flow.
%Moreover, sampling speed of our HBA-Flow model is significantly faster.
Our {HBA-Flow} model outperforms the FloWaveNet model of \cite{kim2019flowavenet} with comparable sampling speeds demonstrating the effectiveness of the coarse-to-fine block autoregressive structure of our HBA-Flow model in capturing long-range spatio-temporal dependencies. This is reflected in the predictive distributions and the top 10\% of predictions of our HBA-Flow model in comparison with FloWaveNet \cite{kim2019flowavenet} in \cref{fig:stanford_drone_dist}. The predictions of our {HBA-Flow} model are more diverse and can more effectively capture the multimodality of the trajectory distributions especially at complex traffic situations \eg intersections and crossings. We provide additional examples in Appendix C. We also observe in \cref{tab:stanford_drone_cross} that the addition of Haar wavelet downsampling \cite{ardizzone2019guided} to FloWaveNets \cite{kim2019flowavenet} (FloWaveNet + HWD) does not significantly improve performance. This illustrates that Haar wavelet downsampling as used in \cite{ardizzone2019guided} is not effective in case of sequential trajectory data as it is primarily a spatial pooling operation for image data. Finally, our ablations with Gaussian priors (HBA-Flow) additionally demonstrate the effectiveness of our HBA-Prior (HBA-Flow + Prior) with improvements with respect to accuracy.
We further include a comparison using the evaluation protocol of \cite{robicquet2016learning,sadeghian2018car,sadeghian2018sophie,deo2019scene} in \cref{tab:stanford_drone_stand}. Here, only a single train/test split is used. We follow \cite{bhattacharyya2019conditional,deo2019scene} and use the minimum average displacement error (mADE) and minimum final displacement error (mFDE) as evaluation metrics. Similar to \cite{bhattacharyya2019conditional,deo2019scene} the minimum is calculated over 20 samples. Our HBA-Flow model outperforms the state-of-the-art demonstrating the effectiveness of our approach.

%\myparagraph{Sampling Speed.} We observe similar sampling speeds of our \emph{HBA-FLow} model to that of FloWaveNets \cite{kim2019flowavenet}. The complex \emph{HBA-Prior} only adds a small overhead ($\sim$10\%). Note that, even with eight flows the autoregressive flows of \cite{kumar2019videoflow} require $\sim$50\% more sampling time in comparison to our \emph{HBA-FLow} model.

\begin{figure}
\centering
\scriptsize
\begin{tabular}{cccc} 
\toprule
Observed & \shortstack{Mean Top 10\%\\ \textcolor[HTML]{4363D8}{B} - GT, \textcolor[HTML]{FFE119}{Y} -\cite{kim2019flowavenet}, \textcolor[HTML]{E6194B}{R} - Ours}   &  \shortstack{FloWaveNet
\cite{kim2019flowavenet}\\ Predictions} & \shortstack{\emph{HBA-Flow} (Ours)\\Predictions} \\
% &  &  &  \\
\midrule
\includegraphics[width=0.2\linewidth]{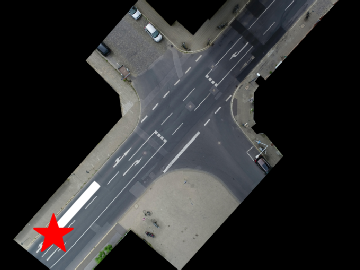} &
\includegraphics[width=0.2\linewidth]{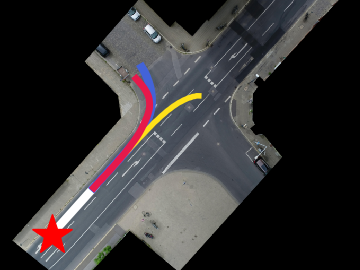} &
\includegraphics[width=0.2\linewidth]{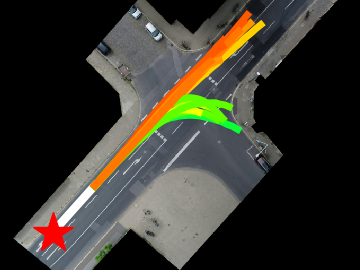} &
\includegraphics[width=0.2\linewidth]{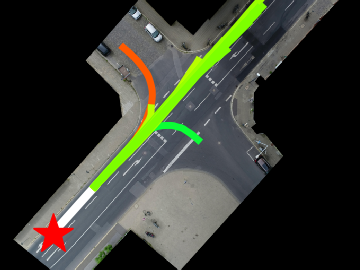}\\

\includegraphics[width=0.2\linewidth]{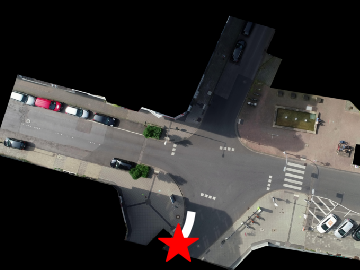} &
\includegraphics[width=0.2\linewidth]{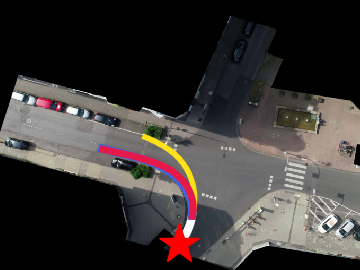} &
\includegraphics[width=0.2\linewidth]{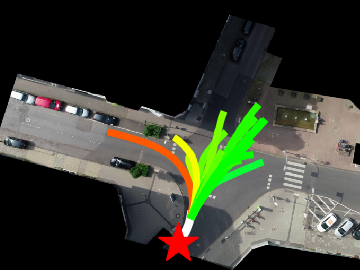} &
\includegraphics[width=0.2\linewidth]{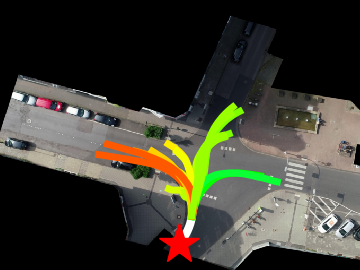}\\

\includegraphics[width=0.2\linewidth]{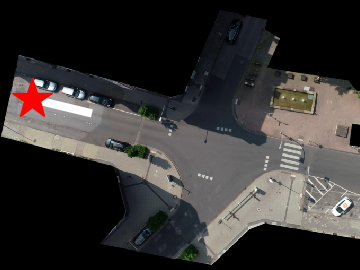} &
\includegraphics[width=0.2\linewidth]{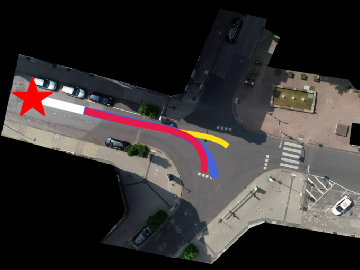} &
\includegraphics[width=0.2\linewidth]{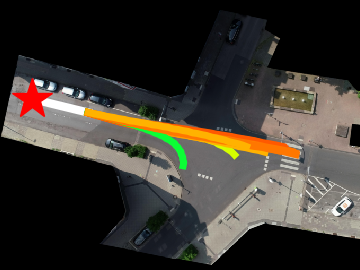} &
\includegraphics[width=0.2\linewidth]{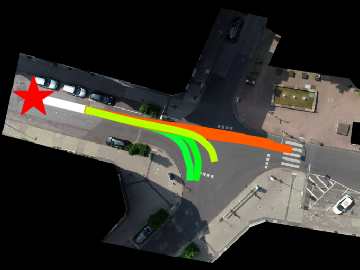}\\

\bottomrule

\end{tabular}
\caption{Mean top 10\% predictions (\textcolor[HTML]{4363D8}{Blue} - Groudtruth, \textcolor[HTML]{FFE119}{Yellow} - FloWaveNet \cite{kim2019flowavenet}, \textcolor[HTML]{E6194B}{Red} - Our \emph{HBA-Flow} model) and predictive distributions on Intersection Drone dataset.
The predictions of our HBA-Flow model are more diverse and better capture the modes of the future trajectory distribution.}
\label{fig:ind_dist}
\end{figure}

\subsection{Intersection Drone}

\begin{table*}[!t]
\centering
\resizebox{\textwidth}{!}{
\begin{tabular}{lc@{\hskip 0.2cm}c@{\hskip 0.2cm}c@{\hskip 0.2cm}c@{\hskip 0.2cm}c@{\hskip 0.2cm}c@{\hskip 0.2cm}c}
\toprule
Method & Er $@$ 1sec & Er $@$ 2sec & Er $@$ 3sec &  Er $@$ 4sec & Er $@$ 5sec & -CLL  \\
\midrule
BMS-CVAE \cite{bhattacharyya2018accurate} &  0.25 & 0.67 & 1.14 & 1.78 & 2.63 & 26.7 \\
CF-VAE \cite{bhattacharyya2019conditional} &  0.24 & 0.55 & 0.93 & 1.45 & 2.21 & 21.2 \\
\midrule
FloWaveNet \cite{kim2019flowavenet} & 0.23 & 0.50 & 0.85 & 1.31 & 1.99 & 19.8 \\
FloWaveNet \cite{kim2019flowavenet} + HWD \cite{ardizzone2019guided} & 0.23 & 0.50 & 0.84 & 1.29 & 1.96 & 19.5 \\
\midrule
{HBA-Flow} + Prior (Ours) & \textbf{0.19} & \textbf{0.44} & \textbf{0.82} & \textbf{1.21} & \textbf{1.74} & \textbf{17.3} \\

\bottomrule
\end{tabular}}
\caption{Five fold cross validation on the Intersection Drone dataset. }
\label{tab:ind_cross}
\end{table*}

We further include experiments on the Intersection Drone dataset \cite{inDdataset}. 
The dataset consists of trajectories of traffic participants recorded at German intersections.
In comparison to the Stanford Drone dataset, the trajectories in this dataset are typically longer. 
Moreover, unlike the Stanford Drone dataset which is recorded at a University Campus, this dataset covers more ``typical'' traffic situations. 
Here, we follow the same evaluation protocol as in Stanford Drone dataset and perform a five-fold cross validation and evaluate up to 5 seconds into the future.

We report the results in \cref{tab:ind_cross}. We use the strongest baselines from \cref{tab:stanford_drone_cross} for comparison to our {HBA-Flow} + Prior model (with our {HBA-Prior}), with three scales, each having eight non-linear squared split coupling flows \cite{ziegler2019latent}. 
For fair comparison, we compare with a FloWaveNet \cite{kim2019flowavenet} model with three levels and eight non-linear squared split coupling flows per level. We again observe that our {HBA-Flow} leads to much better improvement with respect to accuracy over the FloWaveNet \cite{kim2019flowavenet} model. Furthermore, the performance gap between HBA-Flow and FloWaveNet increases with longer time horizons.
This shows that our approach can better encode spatio-temporal correlations.
The qualitative examples in \cref{fig:ind_dist} from both models show that our {HBA-Flow} model generates diverse trajectories and can better capture the modes of the future trajectory distribution, thus demonstrating the advantage of the block autoregressive structure of our HBA-Flow model. 
We also see that our {HBA-Flow} model outperforms the CF-VAE model \cite{bhattacharyya2019conditional}, again illustrating the advantage of exact inference.
%Finally, to better capture interactions between traffic participants, we additionally include a Convolutional Social Pooling \cite{deo2018convolutional} module. The output of the module is used to condition the split coupling flows. We see that our \emph{HBA-Flow} model can effectively leverage these social features, leading to improved performance.

\section{Conclusion}
In this work, we presented a novel block autoregressive \emph{HBA-Flow} framework taking advantage of the representational power of autoregressive models and the efficiency of invertible split coupling flow models. 
Our approach can better represent the multimodal trajectory distributions capturing the long range spatio-temporal correlations.
Moreover, the block autoregressive structure of our approach provides for efficient $\mathcal{O}(\log(T))$ inference and sampling. 
We believe that accurate and computationally efficient invertible models that allow exact likelihood computations and efficient sampling present a promising direction of research of anticipation problems in autonomous systems.

%a coarse-to-fine invertible Haar wavelet based spatio-temporal decomposition of trajectories. 
%Based on the wavelet decomposition, we propose 
%In this work we build on normalizing flows to model trajectories of traffic participants. Normalizing flows allow for exact likelihood computations and efficient sampling from the learned distributions and are therefore very useful in (also partially) autonomous systems. We propose a generalized Haar wavelet decomposition of the data signal and show that this fits naturally into the framework of normalizing flows by providing the Jacobian determinant for this invertible transformation. Based on the wavelet decomposition we propose a block autoregressive structure of conditional split coupling flows with block autogregressive conditional priors which is better able to learn the trajectories than previous state of the art approaches. The block autoregressive structure of our approach results in a efficient $\mathcal{O}(log(N))$ inference and sampling algorithm that is able to capture long range correlations at multiple signal resolution levels and provides state of the art performance. 
\clearpage

\bibliographystyle{splncs03}
\bibliography{egbib}

%===============================================================================

\appendix
\clearpage

\section*{Appendix A. Additional Details of Lemma 1}
\subsection{Proof of Lemma 1}
\begin{lemma}\label{lemma1}
The generalized Haar transformation $f_\textit{hba} = f_\textit{haar} \circ f_\textit{eo}$ is invertible for $\alpha \in [0,1)$ and the determinant of the Jacobian of the transformation $f_\textit{hba} = f_\textit{haar} \circ f_\textit{eo}$ for sequence of length $T_{k}$ with ${\mathbf{y}}^j_k \in \mathbb{R}^d$ is  $\det  \mathbf{J}_\textit{hba} = (1 - \alpha)^{\nicefrac{(d \cdot T_k)}{2}}$. 
\end{lemma}
\begin{proof}
To compute the Jacobian of $f_{hba}$, note that each element of the output fine ($\mathbf{f}_k$) and coarse ($\mathbf{c}_k$) trajectories can be expressed in terms of the elements of the input trajectory $\mathbf{y}_k$. From Eqs.\ (3) and (4) in the main paper, the coarse ($\mathbf{c}_k$) trajectories at level $k$ can be expressed as,

\begin{equation}\label{eq:c}%\tag{5}
\begin{aligned}
    \mathbf{c}_k &= (1 - \alpha) \mathbf{e}_k + 
\alpha \mathbf{o}_k\\
    &= (1 - \alpha) \cdot [{\mathbf{y}}^2_k, \cdots,{\mathbf{y}}^{T_k}_k ] + \alpha \cdot [{\mathbf{y}}^1_k, \cdots,{\mathbf{y}}^{T_k-1}_k ]\\ 
    &= [ \alpha {\mathbf{y}}^1_k + (1 - \alpha) {\mathbf{y}}^2_k, \alpha {\mathbf{y}}^3_k + (1 - \alpha) {\mathbf{y}}^4_k, \cdots,  \alpha {\mathbf{y}}^{T_k-1}_k + (1 - \alpha) {\mathbf{y}}^{T_k}_k  ].
\end{aligned}
\end{equation}

Similarly, the fine ($\mathbf{f}_k$) trajectories at level $k$ can be expressed as,

\begin{equation}\label{eq:f}
\begin{aligned}
    \mathbf{f}_k =& (1 - \alpha) \mathbf{o}_k +  (\alpha - 1) \mathbf{e}_k\\
    =& (1 - \alpha) \cdot [{\mathbf{y}}^1_k, \cdots,{\mathbf{y}}^{T_k-1}_k ] + (\alpha - 1) \cdot [{\mathbf{y}}^2_k, \cdots,{\mathbf{y}}^{T_k}_k ] \\ 
    =& [ (1 - \alpha) {\mathbf{y}}^1_k + (\alpha - 1) {\mathbf{y}}^2_k, (1 - \alpha) {\mathbf{y}}^3_k + (\alpha - 1) {\mathbf{y}}^4_k, \cdots, 
    \\  & (1 - \alpha) {\mathbf{y}}^{T_k-1}_k + (\alpha - 1) {\mathbf{y}}^{T_k}_k  ].
\end{aligned}
\end{equation}

%Therefore, the trajectory after transformation using $f_{hba}$ can be expressed in terms of the elements of the input trajectory $\mathbf{y}_k$. 
We can now rearrange the elements of the output trajectory $f_{hba}$ by placing elements from  $\mathbf{f}_k$ and $\mathbf{c}_k$ in an alternating fashion,

\begin{equation}\label{eq:hba}
\begin{aligned}
f_{hba}(\mathbf{y}_k) = \mathbf{f}_k, \mathbf{c}_k = [ (1 - \alpha) {\mathbf{y}}^1_k + (\alpha - 1) {\mathbf{y}}^2_k, \,\, \alpha {\mathbf{y}}^1_k + (1 - \alpha) {\mathbf{y}}^2_k, \,\, \cdots \,\, ,\\ (1 - \alpha) {\mathbf{y}}^{T_k -1}_k + (\alpha - 1) {\mathbf{y}}^{T_k}_k, \,\, \alpha {\mathbf{y}}^{T_k -1}_k + (1 - \alpha) {\mathbf{y}}^{T_k}_k ].
\end{aligned}
\end{equation}

As each element ${\mathbf{y}}^j_{k} \in \mathbb{R}^d$, we can further simplify the output trajectory $f_{hba}$ in terms of the individual elements in ${\mathbf{y}}^j_k$.
This results in a block diagonal Jacobian $\mathbf{J}_\textit{hba} \in \mathbb{R}^{d\cdot T_k \times d\cdot T_k}$ of $f_\textit{hba}$ of the form,

\begin{equation}
\begin{aligned}
\mathbf{J}_{hba} = \begin{pmatrix} 
(1 - \alpha) & (\alpha - 1) & 0 & 0 & 0 & \cdots & 0 & 0\\
\alpha & (1 - \alpha) & 0 & 0 & 0 & \cdots & 0 & 0\\
0 & 0 & (1 - \alpha) & (\alpha - 1) &  0 & \cdots & 0 & 0\\
0 & 0 & \alpha & (1 - \alpha) & 0 & \cdots & 0 & 0\\
\vdots & \vdots & \vdots & \vdots & \vdots & \ddots & \vdots & \vdots\\
0 & 0 & 0 & 0 & 0 & \cdots & (1 - \alpha) & (\alpha - 1) \\
0 & 0 & 0 & 0 & 0 & \cdots  & \alpha & (1 - \alpha) &\\
\end{pmatrix} .
\end{aligned}
\end{equation}

The repeating block in $\mathbf{J}_\textit{hba}$ repeats $\nicefrac{(d \cdot T_k)}{2} $ times  as the trajectory is of length $T_k$ and each element of the trajectory has $d$ dimensions. Therefore, the determinant of the Jacobian $\mathbf{J}_\textit{hba}$ is $(1 - \alpha)^{\nicefrac{( d \cdot T_k )}{2}}$.

To show that $f_\textit{hba} = f_\textit{haar} \circ f_\textit{eo}$ is invertible, first note that $f_\textit{eo}$ rearranges the elements of the input trajectory as is thus trivially invertible.
Now, note that $f_{haar}$ is a linear system. For $\alpha \in [0,1)$ we see that $\det \mathbf{J}_{hba} > 0$. Thus, the linear system $f_{haar}$ in Eq.\ (4) in the main paper is non-singular and invertible. Thus, $f_\textit{hba}$ is invertible.
\end{proof}

\section*{Appendix B. Architecture and Optimization}
Here, we provide additional architectural details of our {HBA-Flow} model in Fig. 2 (right), in particular the split coupling flows. The split coupling flows in our {HBA-Flow} model are based on those of FloWaveNet \cite{kim2019flowavenet}. However, as mentioned in the main paper, we employ more powerful non-linear squared flows \cite{ziegler2019latent} across baselines versus the affine flows used in \cite{kim2019flowavenet}. The non-causal wavenets in the split coupling flows are similar to the ones employed in \cite{kim2019flowavenet} with 4 convolutional layers with 256 filters each. In practice, we do not find it necessary to employ activation normalization layers along with the more powerful non-linear squared flows. We use identical non-causal wavenets to learn the parameters of our {HBA-Prior}. 

Finally, note that we train the full {HBA-Flow}  model along with the prior using the {AdaMax} \cite{kingma2014adam} optimizer. The ``mixing'' parameter $\alpha$ in $f_\textit{hba}$ is learnable, although $\alpha=0.5$ also works well in practice. 

\section*{Appendix C. Qualitative Results}
We provide additional qualitative results on Stanford Drone in \cref{fig:stanford_drone_dist} and Intersection Drone in \cref{fig:ind_dist} comparing to FloWaveNet \cite{kim2019flowavenet}. These results further support the results in Figs.\ 4 and 5 in the main paper. We again see that the predictions of our HBA-Flow model are more diverse and can more effectively capture the modes of the trajectory distributions at complex traffic situations like intersections and crossings. Again, this is further supported by the top 10\% of predictions, which are closer to the groundtruth trajectories.

\begin{figure}[h]
\centering
\scriptsize
\begin{tabular}{cccc} 

\toprule
Observed & \shortstack{Mean Top 10\%\\ \textcolor[HTML]{4363D8}{B} - GT, \textcolor[HTML]{FFE119}{Y} -\cite{kim2019flowavenet}, \textcolor[HTML]{E6194B}{R} - Ours}   &  \shortstack{FloWaveNet
\cite{kim2019flowavenet}\\ Predictions} & \shortstack{\emph{HBA-Flow} (Ours)\\Predictions} \\
\midrule
\includegraphics[width=0.22\linewidth]{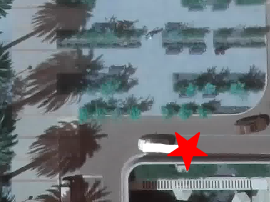} &
\includegraphics[width=0.22\linewidth]{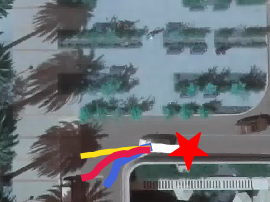} &
\includegraphics[width=0.22\linewidth]{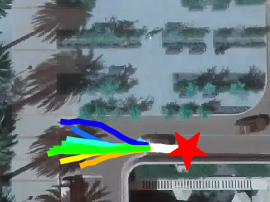} &
\includegraphics[width=0.22\linewidth]{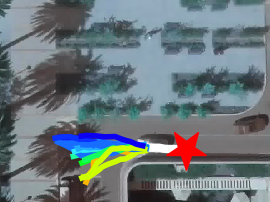}\\

\includegraphics[width=0.22\linewidth]{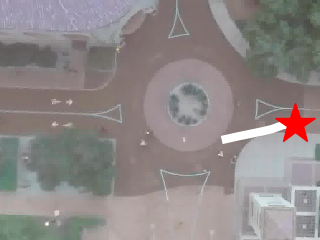} &
\includegraphics[width=0.22\linewidth]{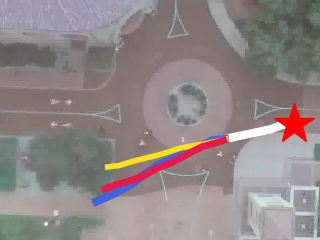} &
\includegraphics[width=0.22\linewidth]{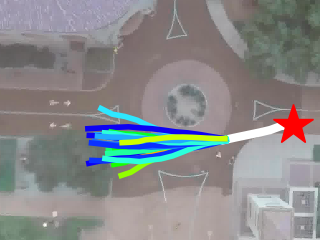} &
\includegraphics[width=0.22\linewidth]{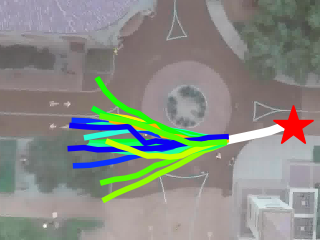}\\

\includegraphics[width=0.22\linewidth]{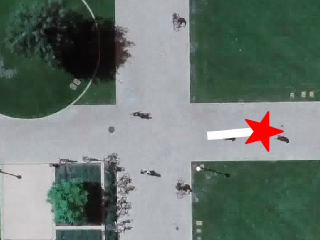} &
\includegraphics[width=0.22\linewidth]{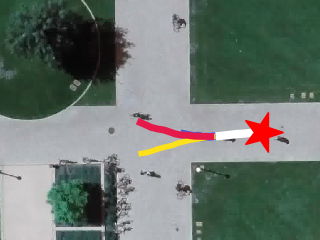} &
\includegraphics[width=0.22\linewidth]{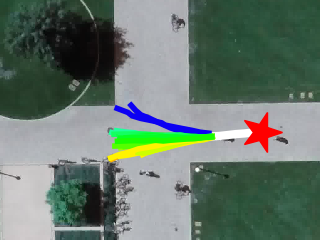} &
\includegraphics[width=0.22\linewidth]{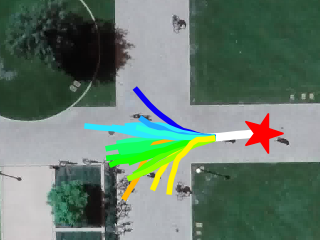}\\

\includegraphics[width=0.22\linewidth]{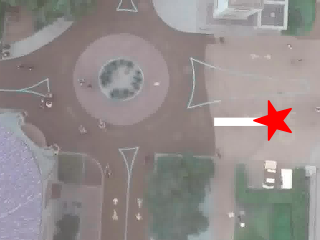} &
\includegraphics[width=0.22\linewidth]{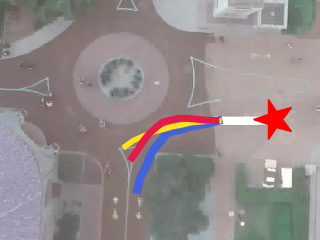} &
\includegraphics[width=0.22\linewidth]{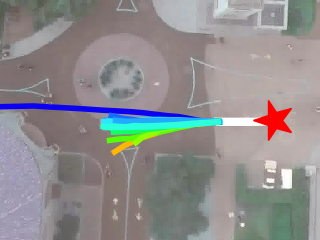} &
\includegraphics[width=0.22\linewidth]{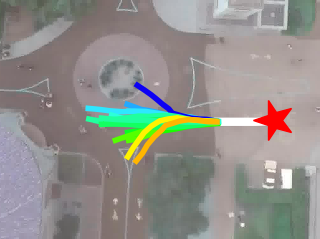}\\

\bottomrule

\end{tabular}
\caption{Mean top 10\% predictions (\textcolor[HTML]{4363D8}{Blue} - Groudtruth, \textcolor[HTML]{FFE119}{Yellow} - FloWaveNet \cite{kim2019flowavenet}, \textcolor[HTML]{E6194B}{Red} - Our \emph{HBA-Flow} model) and predictive distributions on Stanford Drone dataset.
The predictions of our {HBA-Flow} model are more diverse and better capture the modes of the future trajectory distribution.}
\label{fig:stanford_drone_dist}
\end{figure}

\begin{figure}[h]
\centering
\scriptsize
\begin{tabular}{cccc} 
\toprule
Observed & \shortstack{Mean Top 10\%\\ \textcolor[HTML]{4363D8}{B} - GT, \textcolor[HTML]{FFE119}{Y} -\cite{kim2019flowavenet}, \textcolor[HTML]{E6194B}{R} - Ours}   &  \shortstack{FloWaveNet
\cite{kim2019flowavenet}\\ Predictions} & \shortstack{\emph{HBA-Flow} (Ours)\\Predictions} \\
% &  &  &  \\
\midrule
\includegraphics[width=0.22\linewidth]{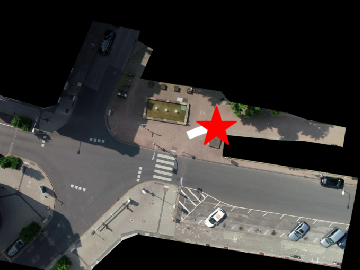} &
\includegraphics[width=0.22\linewidth]{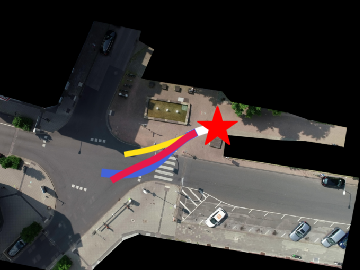} &
\includegraphics[width=0.22\linewidth]{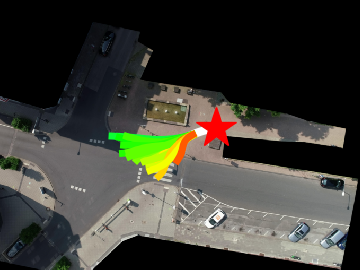} &
\includegraphics[width=0.22\linewidth]{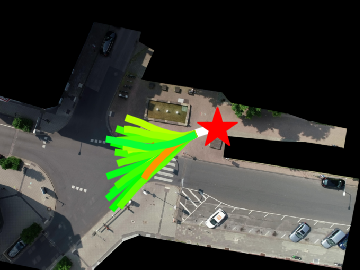}\\

\includegraphics[width=0.22\linewidth]{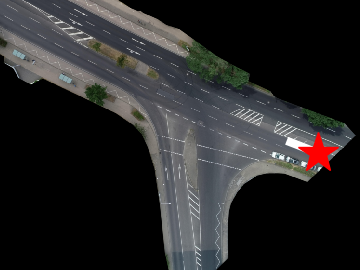} &
\includegraphics[width=0.22\linewidth]{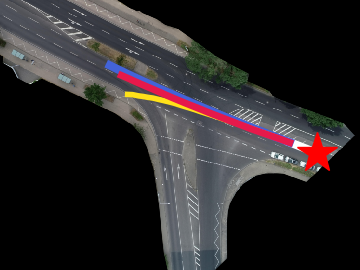} &
\includegraphics[width=0.22\linewidth]{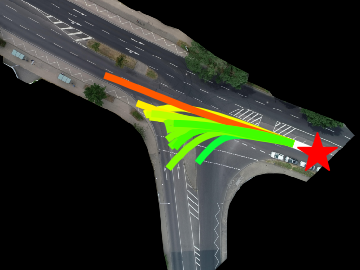} &
\includegraphics[width=0.22\linewidth]{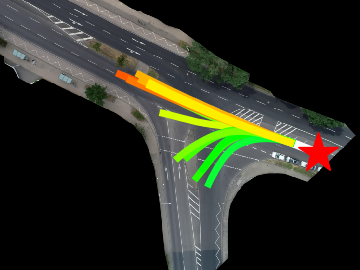}\\

\includegraphics[width=0.22\linewidth]{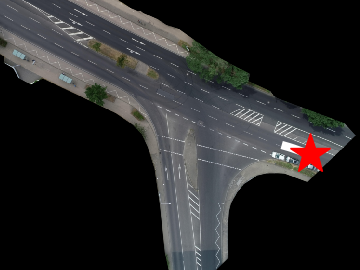} &
\includegraphics[width=0.22\linewidth]{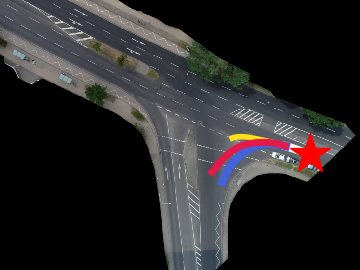} &
\includegraphics[width=0.22\linewidth]{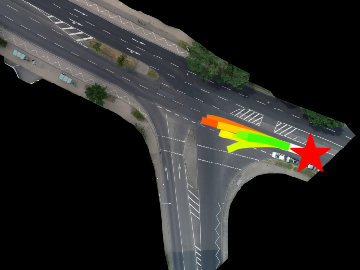} &
\includegraphics[width=0.22\linewidth]{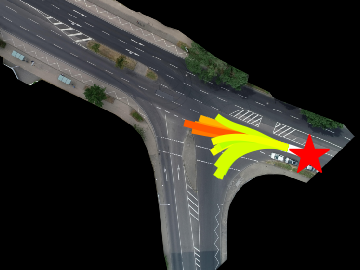}\\

\includegraphics[width=0.22\linewidth]{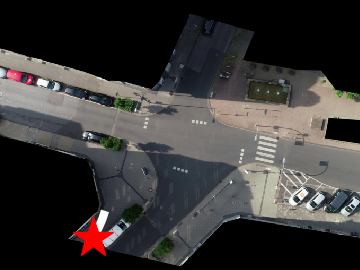} &
\includegraphics[width=0.22\linewidth]{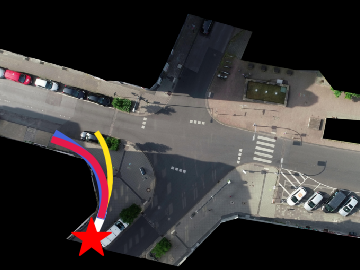} &
\includegraphics[width=0.22\linewidth]{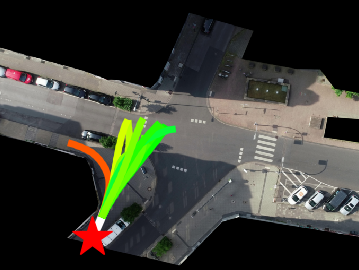} &
\includegraphics[width=0.22\linewidth]{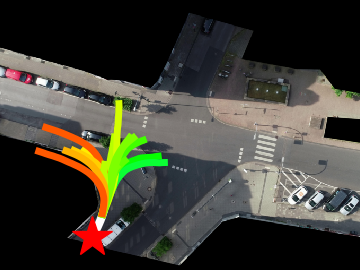}\\

\bottomrule

\end{tabular}
\caption{Mean top 10\% predictions (\textcolor[HTML]{4363D8}{Blue} - Groudtruth, \textcolor[HTML]{FFE119}{Yellow} - FloWaveNet \cite{kim2019flowavenet}, \textcolor[HTML]{E6194B}{Red} - Our \emph{HBA-Flow} model) and predictive distributions on Intersection Drone dataset.
The predictions of our {HBA-Flow} model are more diverse and better capture the modes of the future trajectory distribution.}
\label{fig:ind_dist}
\end{figure}

\end{document}